\documentclass[11pt]{article}

\usepackage{amsmath, amsthm, amssymb}
\usepackage{nicefrac}
\usepackage{xcolor}

\usepackage{microtype}
\usepackage{graphicx}
\usepackage{multirow}
\usepackage{booktabs} 
\usepackage{algorithm}
\usepackage{algorithmic}

\usepackage{caption}
\usepackage{subcaption}
\usepackage{float}
\usepackage{breakcites}

\usepackage{hyperref}

\newtheorem{theorem}{Theorem}

\usepackage{pifont}

\PassOptionsToPackage{numbers, compress}{natbib}

\usepackage{fullpage}

\usepackage[indent]{parskip}
\usepackage[utf8]{inputenc} 
\usepackage[T1]{fontenc}    
\usepackage{hyperref}       
\hypersetup{
    colorlinks=true,
    linkcolor={blue},
    citecolor={red},
    urlcolor={blue},
    breaklinks=true,
	plainpages=true
}
\usepackage{url}            
\usepackage{booktabs}       
\usepackage{amsfonts}       
\usepackage{nicefrac}       
\usepackage{microtype}      
\usepackage{xcolor}         

\usepackage{enumitem}

\title{On the Fine-Grained Hardness of Inverting Generative Models}

\author{%
  Feyza Duman Keles, Chinmay Hegde\thanks{FDK and CH are with the Tandon School of Engineering at New York University. This work was supported in part by the National Science Foundation (under grants CCF-2005804) and
  USDA/NIFA (under grant 2021-67021-35329).} \\
  New York University\\
  \texttt{\{fd2153@nyu.edu, chinmay.h\}@nyu.edu} \\
}

\begin{document}

\maketitle

\begin{abstract}

    The objective of generative model inversion is to identify a size-$n$ latent vector that produces a generative model output that closely matches a given target. This operation is a core computational primitive in numerous modern applications involving computer vision and NLP. However, the problem is known to be computationally challenging and NP-hard in the worst case. This paper aims to provide a fine-grained view of the landscape of computational hardness for this problem. We establish several new hardness lower bounds for both exact and approximate model inversion. In exact inversion, the goal is to determine whether a target is contained within the range of a given generative model. Under the strong exponential time hypothesis (SETH), we demonstrate that the computational complexity of exact inversion is lower bounded by $\Omega(2^n)$ via a reduction from $k$-SAT; this is a strengthening of known results. For the more practically relevant problem of approximate inversion, the goal is to determine whether a point in the model range is close to a given target with respect to the $\ell_p$-norm. When $p$ is a positive odd integer, under SETH, we provide an $\Omega(2^n)$ complexity lower bound via a reduction from the closest vectors problem (CVP). Finally, when $p$ is even, under the exponential time hypothesis (ETH), we provide a lower bound of $2^{\Omega (n)}$ via a reduction from Half-Clique and Vertex-Cover.

\end{abstract}

\newpage

\section{Introduction}

\textbf{Motivation.} Generative models have gained particular prominence in modern machine learning and artificial intelligence, spanning application domains including natural language processing, image processing, and computer vision. At a very abstract level, a generative model is a \emph{trained} machine learning model $G(\cdot)$ (such as a deep neural network with ReLU activations) that takes in as input a latent vector $z$ to produce an output vector $x$:
\begin{equation}
x = G(z) \,, \label{eq:exact}
\end{equation}
where the dimensions and types of $x$ and $z$ depend on the application. Two examples are as follows. If $G(\cdot)$ is the generator of a generative adversarial network (GAN)~\cite{gan}, the latent vector typically corresponds to a noise vector sampled from a random distribution, and the target is a photorealistic image. If $G(\cdot)$ is a generative text-to-image diffusion model (such as Stable Diffusion \cite{LDM} and DALL-E 2 \cite{unCLIP}), the latent vector corresponds to a vector representation of a text string, and the output is an image corresponding to the text. In this context, we study a natural problem that we call \emph{generative model inversion}, defined as follows:

\begin{quote}
\noindent{\textbf{Problem (informal):}} Given a generative model $G : \mathbb{R}^n \rightarrow \mathbb{R}^q$ and a target $x \in \mathbb{R}^q$, find a latent vector $z$ that (either exactly or approximately) produces $x$. 
\end{quote}

This problem has attracted significant attention in recent years across a 
plethora of applications in computer science, such as imaging, denoising,  deconvolution, and compressed sensing. See~\cite{gan_inversion_survey} for a comprehensive overview from the perspective of imaging. 

The widespread use of this problem in applications bring forth the urgency to rigorously study this problem from a theoretical standpoint. In essence, the core computational question can be framed as: given a generative model, can one provide \emph{provable algorithms} for deciphering the latent structure that gives rise to a particular target? In the theory literature, much of the academic focus has been on understanding this question under certain restrictive assumptions. Early works in this direction are the papers~\cite{hand2019global,huang2021provably}. These prove that in the \emph{realizable} case, and where the generative model corresponds to a (shallow, expansive) neural network with random Gaussian weights, polynomial-time inversion is feasible using a gradient-based approach. Follow-up works such as~\cite{latorre2019fast,nguyen2022provable} relax these randomness requirements, but require more sophisticated algorithms and still operate under hard-to-check deterministic assumptions. 

What can be said about the {computational complexity} for inverting models {for general trained networks}? For neural networks with ReLU nonlinearities, the range of the model $G$ can be viewed as the union of an exponential number of subspaces (see Appendix of~\cite{bora2017CS}), suggesting that the problem is inherently hard. Indeed, the only rigorous lower bound we are aware of is the work of Lei et al.~\cite{lei2019inverting}, who prove that for two-layer ReLU networks, the problem of deciding whether or not a particular target $x$ lies within its range is NP-hard, via a reduction from 3SAT.

Notably, in the vast majority of practical applications, achieving an exact inversion is not a stringent requirement. Rather, an approximate solution often suffices. The computational complexity of this softer, yet practically more relevant problem, is a  untapped domain from the theory perspective.

\textbf{Our contributions.} The goal of this paper is to shed further light on the landscape of algorithms and computational complexity related to the inversion of generative models. We provide several new results that leverage recent advances in the machinery of \emph{fine-grained complexity}~\cite{abboud2015tight,rubinstein2019seth,bringmann2021fine}.

Our first contribution strengthens the current state of the art in terms of lower bounds. We establish a fine-grained hardness result for the exact inversion of generative models, conditional on the Strong Exponential Time Hypothesis (SETH). This result serves as an extension of the NP-hardness of this problem proven by Lei et al~\cite{lei2019inverting}.

Our second (and main) contribution is to establish the first known lower bounds on the hardness of approximate generative model inversion, which to our knowledge has attracted little attention in the literature. Concretely, we study (additive) $\varepsilon$-approximate solutions to the following optimization problem:
\begin{equation}
\min_z \|x - G(z) \|_p \,, \label{eq:lpapprox}
\end{equation}
for which we present two distinct sets of fine-grained hardness results. For the rest of this paper, we exclusively focus on the case where $G$ is a deep neural network with ReLU activations; this is the most common family of generative models currently found in practical applications.

For approximate inversion under the \( \ell_p \)-norm where \( p \in [1,3,5,\ldots] \), when $G$ is a width-$n$ network we show an \( \Omega(2^n) \) complexity lower bound, conditional on SETH. 

The practically more prevalent case of approximation under the \( \ell_2 \) norm is somewhat more challenging \cite{aggarwal2022couldnt}. Under the assumption that the Exponential Time Hypothesis (ETH) holds, we present a slightly weaker, \( 2^{\Omega(n)} \), complexity lower bound.

\textbf{Techniques.} To accomplish these contributions, we employ a variety of fairly intuitive techniques. The bulk of our techniques are inspired by a set of results recently developed for the fine-grained hardness of the \( k \)-sparse linear regression (\( k \)-SLR) problem~\cite{gupte2021fine}:
$$
\min_z \| y - Az \|_p, \quad \text{s.t.}~\|z\|_0 \leq k.
$$
where $y \in \mathbb{R}^M$ is a target and $A \in \mathbb{R}^{M \times N}$ is a design matrix. For this $\ell_0$-constrained optimization problem, the authors show that under popular conjectures from fine-grained complexity, there exists no algorithm for this problem with running time $N^{(1-\varepsilon)k}$ for any $\varepsilon > 0$. However, our setting of generative model inversion is qualitatively quite different from $k$-SLR, and we have to make non-trivial changes to their proof techniques to succeed in our case.

In the setting of exact inversion, our improved lower bound is predicated on a reduction from the \( k \)-SAT problem, which strengthens the result in  \cite{lei2019inverting} provided SETH holds. 

In the setting of approximate inversion, our techniques are multi-faceted. For odd $\ell_p$ norms, our lower bound results stem from a reduction from the closest vector problem (CVP). In contrast, for $p=2$, we derive two lower bounds conditional on ETH, respectively through reductions from the Half-Clique and Vertex Cover problems.

\textbf{Implications.} Exact inversion was already known to be hard due to~\cite{lei2019inverting}, so our first result should not be too surprising. However, the complexity of approximate inversion was not known, and we believe that our main results may serve as lamp posts for (provable) algorithm design for this area. Moreover, we believe that the connection to fine-grained bounds for sparse linear regression~\cite{gupte2021fine} is interesting and may lead to further clarity of the fundamental complexities of these two important families of problems in algorithmic learning theory.

\section{Related work}
\label{sec:related}

\textbf{Generative model inversion.} 
Recovery of latent vectors that generate a given target has attracted significant attention over the last three decades. The focus has shifted away from ``linear'' generative models such as sparse models~\cite{bpdn,cosamp,modelcs} and towards nonlinear generative models such as convolutional neural networks \cite{onenet,mousavi2017learning}, pre-trained generative priors \cite{CSGAN}, or untrained deep image priors \cite{DIP,netgd,deepcs}. 

However, progress on the theoretical side has been more modest. On the one hand, the seminal work of~\cite{CSGAN} established the first {statistical} upper bounds (in terms of measurement complexity) for compressed sensing with generative models. These bounds have been shown in~\cite{scarlett2020} to be nearly optimal. On the other hand, provable \emph{algorithmic} upper bounds for generative inversion are only available in restrictive cases.  The paper~\cite{shah2018solving} proves the convergence of projected gradient descent (PGD) for compressed sensing with generative priors under the assumption that the range of the (deep) generative model $G$ admits a polynomial-time projection oracle. This assumption is equivalent to showing that~\eqref{eq:lpapprox} can be solved in polynomial time for $p=2$. 

A (limited) number of papers have attempted to establish upper bounds on either~\eqref{eq:exact} or~\eqref{eq:lpapprox}. The paper~\cite{hand2019global} proves the convergence of (a variant of) gradient descent for shallow generative priors whose weights obey a distributional assumption. The paper~\cite{lei2019inverting} shows the correctness of an intuitive layer-wise inversion algorithm for sufficiently expansive networks, as well as establishes NP-hardness lower bounds in the general case.

Perhaps the most general algorithmic result in this line of work is by~\cite{latorre2019fast}. There, the authors show that under certain structural assumptions on $G$, a linearized alternating direction method of multipliers (ADMM) applied to a regularized mean-squared error loss converges to a neighborhood of $x^*$. In subsequent work, \cite{cslangevin} show that Langevin Markov Chain Monte Carlo (MCMC) also converges (in expectation) to a neighborhood of the true solution, under similar assumptions. However, these assumptions are somewhat hard to verify in practice.

More recently, \cite{whang2020compressed,asim2019invertible} have advocated using \textit{invertible} generative models, which use real-valued non-volume preserving (NVP) transformations \cite{dinh2016density}. An alternate strategy for sampling images consistent with linear forward models was proposed in \cite{lindgren2020conditional} where the authors assume an invertible generative mapping and sample the latent vector $z$ from a second generative invertible prior. In contrast, our focus is on more generic families of generative neural whose layer-wise forward mappings involve ReLU activations and are not necessarily invertible. 

\textbf{Fine-grained complexity.}  Classical complexity theory has traditionally attempted to delineate the boundary between problems that admit efficient (polynomial-time) algorithm and problems that do not. A fine(r) grained picture of the landscape of polynomial-time has begun to emerge over the last decade. In particular, the focus has shifted towards pinning down the {exponent}, $c$, of a problem that can be solved in polynomial time $\tilde{O}(n^c)$. Similar to NP-Hardness results, most of these newer results are conditional and rely on reductions from popular (but plausible) conjectures such as the Strong Exponential Time Hypothesis (SETH)~\cite{seth}, \cite{seth2}. See the relevant surveys \cite{indyk2017beyond}, \cite{lokshtanov2013lower}, \cite{rubinstein2019seth}, and \cite{bringmann2021fine} for comprehensive overviews of this emerging area. In particular, this approach has been shown to provide conditional lower bounds on well-known problems such as edit distance \cite{backurs2015edit}, Frechet distance \cite{bringmann2014walking}, dynamic time warping \cite{bringmann2015quadratic}, longest common subsequence (LCS) \cite{abboud2015tight}, and string matching \cite{abboud2018if}.

In the context of machine learning, reductions from SETH have been fruitfully applied to problems such as clustering \cite{abboud2019subquadratic}, kernel PCA \cite{backurs2017fine}, sparse linear regression \cite{gupte2021fine}, Gaussian kernel density estimation~\cite{aggarwal2022optimal}, and approximate nearest neighbors \cite{rubinstein2018hardness}. In very recent work, this approach has also been shown to imply an $\Omega(n^2)$-lower bound for transformer models with input size $n$~\cite{keles2023computational}.

\section{Notations and Preliminaries}\label{sec:prelims}

Throughout the paper, we use upper case characters to denote matrices, while lower case characters for vectors. We use $A_i$ to denote the $i^{th}$ row of matrix $A$. We use $A_{ij}$ to denote the element at the $i^{th}$ row and $j^{th}$ column of matrix $A$. Also, we use $a_i$ to denote the $i^{th}$ row of vector $a$. $\Vec{1}$ represents the all-ones vector, while $\Vec{0}$ represents the all-zeroes vector. For a positive integer $n \in \mathbb{Z}^+$, $[n]$ denotes the set of all positive integers up to $n$, and we used $a_{[n]}$ to call the first $n$ entries of vector $a$. 

\subsection{Generative Model Inversion}

A 1-layer ReLU network $G_1: \mathbb R^n \to \mathbb R^{m_1} $ can be defined as $G_1(z)=\textnormal{ReLU}(W_1z+b_1)$ with weight matrix $W_1\in \mathbb R^{m_1 \times n}$ and bias $b_1 \in \mathbb R^{m_1}$, where $m_1$ is the number of hidden neurons. An $L$-layer neural network $G:=G_L$ can be expressed with the following recurrence relation:
\begin{align}\label{eq_relu_network}
    G_{l}(z)=\textnormal{ReLU}(W_lG_{l-1}(z)+b_l) \textnormal{ for }  l \in \mathbb Z^+ ,\ l = 2,3,\ldots,L, 
\end{align}
where $W_l \in \mathbb R^{m_{l} \times m_{l-1}}$ are the layer-wise weight matrices, and $b_l \in \mathbb R^{m_l}$ are layer-wise bias vectors. We assume that the network width, $\max_l m_l$, is bounded as $O(n)$ unless otherwise specified.

In the generative inversion problem, we are given a ReLU network $G$ and an observation $x$. Then, the purpose is to determine the closest point of the range of the neural network to the input $x$. Mathematically, we want to find $z^*$ that satisfies the following under a given norm: 
$$z^*= \arg\min_{z} \|G(z)-x\| .$$ 
Stated as a decision problem, the goal of exact recovery is to determine if $x$ lies within the range of $G$. In other words, distinguish between two cases: 
\begin{itemize}[nosep,leftmargin=*]
\item YES: either there exists a $z^*$ such that $G(z^*)=x$, 
\item NO: for all $z$ we have $G(z)\neq x$. 
\end{itemize}
This problem is norm-independent because, under any norm, $\|G(z)-x\|=0$ whenever $G(z)=x$.

On the other hand, in many practical cases instead of having an exact solution, finding a close enough point to the range of the network suffices.
Therefore, we investigate the hardness of the decision problem where the goal is to distinguish between the two cases for a parameter $\delta > 0$: 
\begin{itemize}[nosep,leftmargin=*]
\item YES: there exists a point $z^*$ such that $\|G(z^*)-x\|<\delta$, 
\item NO: for all $z$ we have $\|G(z)- x\|\geq \delta$. 
\end{itemize}

The choice of the norm $\|\cdot\|$ will play a significant role. Indeed, we prove different hardness results when $\|\cdot\|$ corresponds to $\ell_p$-norm depending on the parity of $p$.

\subsection{(Strong) Exponential Time Hypothesis and $k$-SAT}

Given a SAT formula on $n$ variables with each clause of size $k$, the $k$-SAT problem is to distinguish between the cases whether the formula is satisfiable, or whether it is not. Despite decades of effort, no one has invented a faster-than-exponential ($O(2^n)$) time algorithm for this problem. In fact, unless $P=NP$, no polynomial-time algorithm exists. 

The Strong Exponential Time Hypothesis (SETH) is a strengthening of this statement~\cite{seth}: for every $\varepsilon > 0$, there is no (randomized) algorithm that solves $k$-SAT in $2^{(1-\varepsilon)n}$ time. The Exponential Time Hypothesis (ETH) is a (slightly) weaker conjecture: there exists $\delta > 0$ such that 3-SAT cannot be solved in time $2^{\delta n}$. 

\subsection{Closest Vector Problem}

In some of our proofs, we leverage SETH-hardness of the closest vector problem (CVP), a central problem in cryptography and complexity theory defined as follows. Consider a lattice $\mathcal{L}=\mathcal{L}(\mathbf{B})=\{\mathbf{B}z \; | \; z \in \mathbb Z ^n\}$, where $B=(\mathbf v_1,\mathbf v_2,\dots,\mathbf v_n)$ and all the vectors $\mathbf v_i \in \mathbb R ^d$. Given a target vector $\mathbf t$, define $dist_p(\mathcal{L},\mathbf t) = \min_{\mathbf x \in \mathcal{L}(\mathbf{B})} \|\mathbf x - \mathbf t \|_p$

The goal of the CVP$_p$ problem is to distinguish between two cases:
\begin{itemize}[nosep]
\item YES: $dist_p(\mathcal{L},\mathbf t) \leq r$ and 
\item NO: $dist_p(\mathcal{L},\mathbf t)>r $.
\end{itemize}
A result of~\cite{bennett2017quantitative} states that assuming SETH, then CVP$_p$ cannot be solved in $O(2^{(1-\epsilon)n})$ time for odd integers $p\geq 1, p \neq 2 \mathbb Z$.

As follow-up work, \cite{aggarwal2021fine} proved the same hardness for the so-called (0,1)-CVP$_p$ problem: for any $1 \leq p \leq \infty$, distinguish between the two cases: 
\begin{itemize}
    \item YES: $\|By^*-t \| \leq r$ for some $y^*\in \{ 0,1\}^n $ and 
    \item NO: $\|By-t \|>r+\tau $ for all $y\in \mathbb Z$.
\end{itemize}

\subsection{Half-Clique and Vertex Cover}

We also consider two graph optimization problems: Vertex Cover and Half-Clique. The goal in Vertex Cover is to find a set of vertices of minimum size that touch every edge in a given graph; the decision version is to decide whether or not a vertex cover of given size exists. The Half-Clique problem is a special case of the $k$-CLIQUE problem (for $k=n/2$), where the goal is to decide whether or not there exists a clique with $k$-vertices. Both decision problems are classical, NP-complete, and ETH-hard \cite{lokshtanov2013lower}.



\section{Hardness of Exact Inversion}\label{section_exact}
As stated, we are given a ReLU network $G(\cdot)$ and an observation $x$. The problem is to determine whether there exists a $z$ such that $G(z)=x$.  The work \cite{lei2019inverting} shows the NP-hardness of the exact inversion problem. We strengthen this to show SETH-hardness below in \ref{thm_exact_binary}. First, we provide hardness results when $z \in \{ 1,-1 \}^n$, and we generalize these results for  $z \in \mathbb R^n$ in Theorem \ref{thm_exact_real}.

\begin{theorem}
Suppose SETH holds. Then for any $\epsilon>0$ there is a 2-layer, $O(n)$-width ReLU network, $G_2(z)$ and a target $x$ for which no $O(2^{(1-\epsilon)n})$ time algorithm can demonstrate the existence of $z \in \{ 1,-1 \}^n$ that satisfies $G_2(z)=x$.
\label{thm_exact_binary}
\end{theorem}
\begin{proof}
The idea of the proof is a reduction from $k$-SAT. 

Consider the $k$-SAT problem which has $m$-clauses and $n$-literals.  Let $z=[z_1,z_2,\dots,z_n]^T\in\{1,-1\}^n$ where $z_i$ represents the assignment of $i^{th}$ literal. If in the assignment, $i^{th}$ literal takes {TRUE}, then $z_i=1$ and if it takes {FALSE} then $z_i=-1$.


We will construct the first-layer weight matrix $W_1$ as an $m \times n$ matrix whose rows correspond to clauses (call $j^{th}$ row as $(W_1)_j$) and columns correspond to literals in the following sense. If $i^{th}$ literal appears positively in $j^{th}$ clause then put $-1$ to the $i^{th}$ element of the corresponding $j^{th}$ row (in other words $(W_1)_{ji}=-1$). Similarly if a literal appears negatively, then place $1$ to the $i^{th}$ element ($(W_1)_{ji}=1$), and if there is no appearance of this literal, then place $0$ to the $i^{th}$ element ($(W_1)_{ji}=0$). Also, the bias vector of the first layer is constructed as follows: $b_1=[-(k-1), -(k-1), \dots,-(k-1)]^T \in \mathbb R^m$. For the second layer, $W_2$ is all $1$ matrix in $\mathbb R^{1\times m}$, while $b_2=0$. Lastly, the given observation $x \in \mathbb R$ is $0$.

The construction of these matrices takes $O(mn)$-time. We will show that determining whether there is a $z$ for which $G_2(z)=x$ (defined by \eqref{eq_relu_network}) in this construction can be reduced from $k-SAT$.

Denote the $j^{th}$ entry of the multiplication of $W_1z$ is $(W_1z)_j= \langle (W_1)_j,z \rangle =\sum_{i=1}^n (W_1)_{ji}z_i$. If the assignment of the literals satisfies $j^{th}$ clause, then by the definition of the variables $(W_1)_{ji}z_i=-1$ for at least one $i \in [n]$. Since each clause has $k$ literals in $k$-SAT, $\langle (W_1)_j,z \rangle =\sum_{i=1}^n (W_1)_{ji}z_i \leq k-1$ in the satisfied assignment case. Also $(b_1)_j=-(k-1)$, so that $\langle (W_1)_j,z \rangle + (b_1)_j \leq 0$. Because of the ReLU function, $\textnormal{ReLU}\big(\langle (W_1)_j,z \rangle + (b_1)_j \big)= 0$. This means that $j^{th}$ entry of $G_1(z)=\textnormal{ReLU}(W_1z + b_1 )$ is $0$ when the assignment satisfies $j^{th}$ clause.

If the assignment of the literals does not satisfy $j^{th}$ clause, then for all the literals that appear in the clause, we have $(W_1)_{ji}z_i=1$. Since each clause has $k$ literals in $k$-SAT, $\langle (W_1)_j,z \rangle =\sum_{i=1}^n (W_1)_{ji}z_i = k$ in the non-satisfied assignment case. Also $(b_1)_j=-(k-1)$, so that $\langle (W_1)_j,z \rangle + (b_1)_j = 1$. Positive numbers stay the same after ReLU function, so $\textnormal{ReLU}(\langle (W_1)_j,z \rangle + (b_1)_j )= 1$. This means that $j^{th}$ entry of $G_1(z)=\textnormal{ReLU}(W_1z + b_1 )$ is $1$ when the assignment does not satisfy $j^{th}$ clause.

Consider an assignment of the literals that satisfies all the clauses. Then, all the entries of $G_1(z)=\textnormal{ReLU}(W_1z + (b_1)_j )$ are $0$. For the second layer, $W_2$ is all $1$ matrix in $\mathbb R^{1\times m}$ and $b_2=0$, so it is the summation of the entries of $G_1(z)$. Then, when all clauses are satisfied, we have

$$ G_2(z)=\textnormal{ReLU}(W_2 G_1(z) + b_2) = 0 =x \, .
$$

On the other hand, consider an assignment of the literals that does not satisfy at least one clause. Then, at least one of the clauses is not {TRUE}. So, at least one of the entries of $G_1(z)=\textnormal{ReLU}(W_1z + b_j )$ is $1$. As the second layer adds up all the entries, we obtain,

$$ G_2(z)=\textnormal{ReLU}(W_2 G_1(z) + b_2) \geq 1 >x
$$

Hence, there exists a $z^* \in \{1,-1\}^n$ for which $G_2(z^*)=x$ if and only if there is a satisfied assignment to $k$-SAT. This completes the reduction.

As a result, for any $\epsilon>0$, there is a 2-layer ReLU network $G_2$ and observation $x$ such that there is no deterministic $2^{(1-\epsilon)n}$-time algorithm to determine whether there exists a $z$ with $G_2(z)=x$.

\end{proof}

Theorem \ref{thm_exact_binary} holds for the case of binary latent vectors. We extend this to all real valued latent vectors in Theorem \ref{thm_exact_real}. To achieve this, we use an extra two layers to force the input of the third layer to $\{0,1\}^n$ as in Theorem~\ref{thm_exact_binary}.




\begin{theorem}
Suppose SETH holds. Then for any $\epsilon>0$, there is a 4-layer, $O(n)$-width ReLU network $G_4$ and an observation $x$ such that there is no $O(2^{(1-\epsilon)n})$ time algorithm to determine whether there exists $z \in\mathbb R^n$ satisfying $G_4(z)=x$. 
\label{thm_exact_real}
\end{theorem}

\begin{proof}
    We again make a reduction from k-SAT, by modifying the proof of Theorem \ref{thm_exact_binary}. The input is $z=[z_1, z_2, \dots, z_n]^T \in \mathbb R^n$. The first 2 layers will be used to map all the values into the interval $[-1,1]$. The first layer is to bound from below with $-1$, so the output is $\max\{z_i,-1\} $ for all $i \in [n]$. The second layer is to bound from above with $1$, so the output of the second layer is $v \in \mathbb R^n$ such that 
    $v_i=\min\{ \max \{ z_i,-1\},1\}$ for all $i \in [n]$.
    
    For the third layer, the output will be $u \in R^{m+2}$ where first $m$ nodes will be defined as the first layer of Theorem \ref{thm_exact_binary} with same weight and bias, and we will add 2 more nodes $u_{m+1}=\sum_{i=1}^n \max\{v_i,0\}$ and $u_{m+2}=\sum_{i=1}^n -\min \{ v_i,0\}$.

    The last layer has 2 output nodes, the first output node is the summation of the first $m$ input nodes (i.e. $G_4(z)_1=\sum_{i=1}^n u_i$), and the second output node is the summation of the last 2 nodes (i.e. $G_4(z)_2=u_{m+1}+u_{m+2}$). Also suppose the given observation $x=[0,n]^T \in \mathbb R^2$.

    For any real number $a \in [-1,1]$, we have the following inequality; $\max\{a,0\}-\min\{a,0\} \leq 1$, and the equality holds only when $a=-1$ or $a=1$.  Therefore, since $v_i=\min\{ \max \{ z_i,-1\},1\} \in [-1,1]$ for all $i \in [n]$, we have $G_4(z)_2=u_{m+1}+u_{m+2}=\sum_{i=1}^n \max\{v_i,0\}+\sum_{i=1}^n -\min \{ v_i,0\}=\sum_{i=1}^n (\max\{v_i,0\}-\min \{ v_i,0\})\leq n = x_2$, and equality holds when all $v_i$ are $1$ or $-1$. 

    So, the first $m$ nodes of the third layer, the first node of the fourth layer, and $x_1=0$ will create the same structure as in the proof of Theorem \ref{thm_exact_binary}.
    
\end{proof}

By Theorem \ref{thm_exact_binary} and \ref{thm_exact_real}, we can conclude that the exact variant of generative model inversion for deep ReLU networks is not only NP-complete; any algorithm faces an \emph{exponential time} (in terms of the latent vector size) lower bound under the assumption that SETH holds.

Practically speaking, while inverting generative models, it is often sufficient to find a ``good enough'' point rather than exact recovery. Therefore, in the following section, we study the hardness of the existence of an approximate point.

\section{Hardness of Approximate Inversion}

We investigate the hardness of determining whether there exists $z$ such that $\|G(z)-x\|<\delta$ for a given $\delta>0$ and various choices of norms $\|\cdot\|$. 

We first lower bound the complexity of the inverse generative model under the $\ell_p$-norm for positive odd numbers $p$ by reducing it from CVP$_p$. This implies a runtime lower bound of $O(2^n)$ by assuming SETH. On the other hand, for positive even numbers $p$, we reduce the problem from the Half-Clique and Vertex Cover Problems. This implies a (slightly weaker, but still exponential) runtime lower bound of $2^{\Omega(n)}$ by assuming ETH.

\subsection{Reduction from CVP}

In Theorem \ref{thm_approx_binary}, we present hardness results on the input $z \in \{ 0,1 \}^n$, which are derived by applying a reduction from $(0,1)$-CVP$_p$. Our results are inspired by the proofs of the hardness of the sparse linear regression problem established in \cite{gupte2021fine}.  We further extend our findings by proving that the results also hold for $z \in \mathbb R^n$ in Theorem \ref{thm_approx_real}, with a reduction from the binary case.

\begin{theorem}
Assume SETH. Then for any $\epsilon>0$, there is a 1-layer, $O(n)$-width ReLU network $G_1$ and an observation $x$ such that there is no $O(2^{(1-\epsilon)n})$ time algorithm to determine whether there exists $z \in \{ 0,1 \}^n$ that satisfies $\|G_1(z)-x\|_p<\delta$ for a given $\delta>0$ and any positive odd number $p$. 
\label{thm_approx_binary}    
\end{theorem}


\begin{proof}

We will make a reduction to $(0,1)$-CVP$_p$, which is distinguishing between $\|By^*-\mathbf t \|_p \leq r$ for some $y^*\in \{ 0,1\}^n $ and $\|By-\mathbf t \|_p>r $ for all $y\in \mathbb \{ 0,1\}^n$, where  $B=(\mathbf v_1,\mathbf v_2,\dots,\mathbf v_n)$ and all $ \mathbf v_i \in \mathbb R ^d$, and $\mathbf t$ is target vector.

The construction of weights $\overline{W_1} $ and bias $\overline{b_1} $ in the layer of the generative model is as follows:
\begin{align*}
W_1=
\begin{bmatrix}
\mathbf{v}_1 & \Vec{0} & \mathbf v_2 & \Vec{0} & ... & \mathbf v _n & \Vec{0}\\
\alpha & \alpha & 0 & 0 & ... & 0 & 0\\
0 & 0 & \alpha & \alpha & ... & 0 & 0\\
0 & 0 & 0 & 0 & ... & \alpha & \alpha
\end{bmatrix}, \;
b_1=\begin{bmatrix}
- \mathbf t\\
- \alpha \\
- \alpha \\
- \alpha
\end{bmatrix} 
\end{align*}

for some positive real number $\alpha$ to be determined later.

Let $\overline{W_1}=\begin{bmatrix} W_1 \\ -W_1 \end{bmatrix} $ and $\overline{b_1}=\begin{bmatrix} b_1 \\ -b_1 \end{bmatrix} $. 
Thus, the 1-layer generative model is $ G_1(z)= \textnormal{ReLU}(\overline{W_1} z + \overline{b_1}) $
Also, suppose the given observation $x$ is the all-zeros vector and $\delta=r$.

For reduction to $(0,1)$-CVP$_p$ problem, we need to show two things: 1) a YES instance for $(0,1)$-CVP$_p$ implies YES for generative inversion, and 2) a YES for generative implies YES instance for $(0,1)$-CVP$_p$ problem. 

$\implies$ Suppose, for the problem of $(0,1)$-CVP$_p$, there is some $y^* \in \{0,1\}^n$ such that $\|By^*-t \| \leq r$.

Construct the $z^*$ as:
$$z_{2i-1}^*= \begin{cases}
    1, & \text{if } y_i^* = 1\\
    0, & \text{if } y_i^* = 0
    \end{cases} $$
$$z_{2i}^*= \begin{cases}
    0, & \text{if } y_i^* = 1\\
    1, & \text{if } y_i^* = 0
    \end{cases} $$

Then, 
\begin{align*}
    \|G_1(z^*)-x \|_p &= \| \textnormal{ReLU}(\overline{W_1} z^* + \overline{b_1}) \|_p \\
    &= \|W_1z^*+b_1 \|_p \\  &= \|By^*-t\|_p \leq r = \delta
\end{align*}

$\impliedby$ Suppose there is some $z^* \in \{0,1\}^n$ such that $\|G(z^*)-x \|_p < \delta$. Because of ReLU function, $\|\textnormal{ReLU}(\overline{W_1} z^* + \overline{b_1}) \|_p = \| W_1 z^* + b_1 \|_p$. 
For an index $i$, if both of $z_{2i-1}^*,z_{2i}^*$ are $0$ or both are $1$, then  $\| W_1z^*+b_1 \|_p > \alpha > \delta$ for sufficiently large $\alpha$, which contradicts the choice of $z^*$. Therefore, exactly one of $z_{2i-1}^*,z_{2i}^*$ is $0$ and the other is $1$.

Then, construct $y^*$ as follows:
$$ y_i^* = \begin{cases}
    1, & \text{if } z_{2i-1}^* = 1\\
    0, & \text{if } z_{2i-1}^* = 0
    \end{cases} $$

Therefore,
\begin{align*}
    r = \delta  &> \|G_1(z^*)-x \|_p \\ 
    &= \| \textnormal{ReLU}(\overline{W_1} z^* + \overline{b_1}) \|_p \\
    &= \|W_1z^*+b_1 \| \\  &= \|By^*-t\|_p 
\end{align*}

\

This completes the reduction. Since $(0,1)$-CVP$_p$ cannot be solved in $O(2^{(1-\epsilon)n})$-time for any $\epsilon>0$, this also holds for generative inversion when the input is in $\{0,1\}^n$.







\end{proof}

Hardness for the general case $z \in \mathbb R^n$ are proved by reducing the problem from the binary case. In the case of exact inversion, the transition from $\{0,1\}^n$ to $\mathbb R^n$ was accomplished by adding two extra layers. When we try to apply the same trick in the approximate case, instead of having exact binary values $\{0,1\}$, we result in points clustered in a small interval around $0$ or around $1$. To overcome this issue, we introduce two additional layers that map the points around $0$ to $0$ and those around $1$ to $1$. 

\begin{theorem}
Assume SETH. Then for any $\epsilon>0$, there is a 5-layer ReLU network $G_5$ and an observation $x$ such that there is no $O(2^{(1-\epsilon)n})$ time algorithm to determine whether there exists $z \in \mathbb R ^n$ that satisfies $\|G_5(z)-x\|_p<\delta$ for a given $\delta>0$ and any positive odd number $p$. 
\label{thm_approx_real}    
\end{theorem}

The following proof is for $\delta<{1}/{4}$. The proof of the theorem for any $\delta>0$ is given in Appendix \ref{proof_thm4_any_delta} with slight changes.
\begin{proof}
    The first 2 layers are designed to make the entries $z$ in $[0,1]$, for this purpose $z_i$ is mapped to $v_i=\min\{ \max \{ z_i,0\},1\}$ for all $i \in [n]$ as in proof of Theorem \ref{thm_exact_real}.

    In the $3^{rd}$ layer, the first $n$ nodes are defined by the following formula $u_{[n]}=\textnormal{ReLU}(W_3 v + b_3)$ where $W_3=-  I_n $ and $b_3= \frac{1}{2}\cdot\Vec{1} \in \mathbb R^n$. Also, introduce 2 more nodes that $u_{n+1}=\sum_{i=1}^n \max\{v_i,1/2\}$ and $u_{n+2}=\sum_{i=1}^n -\min \{ v_i,1/2\}$.

    In the $4^{th}$ layer, the first $n$ nodes are defined by $t_{[n]}=\textnormal{ReLU}(W_4 u_{[n]} + b_4)$ where $W_4=- 4\cdot I_n $, and $b_4= \Vec{1} \in \mathbb R^n$. Also, we add another node $t_{n+1}=u_{n+1}+ u_{n+2}= \sum_{i=1}^n \max\{v_i,1/2\} + \sum_{i=1}^n -\min \{ v_i,1/2\}$.

    In the last layer, construction is similar to the proof of Theorem \ref{thm_approx_binary} with an addition of one more node given by $s_{m+1}= t_{n+1}$ so that the output is $s \in \mathbb R^{m+1}$.

    The given output $x$ is $x=[\Vec{0},n/2]\in \mathbb R^{m+1}$. 

    We need to show two things for reduction: 1) A YES instance for $(0,1)$-CVP$_p$ implies YES for generative model inversion, and 2) a YES for generative model inversion implies YES for $(0,1)$-CVP$_p$.

    $\implies$ The first part is straightforward as in Theorem \ref{thm_approx_binary}. 
    
    $\impliedby$ For the second part, suppose there is a $z^* \in R^n$ such that $\|G_5(z^*)-x \|_p < \delta$. The first 2 layers bound the input $z$, so we take the corresponding bounded values $v^* \in [0,1]^n$. 
    
    We know that if a real number $a \in [0,1]$, then $\max\{a,1/2\}-\min\{a,1/2\} \leq 1/2$, and the equality holds only when $a=0$ or $a=1$. We are given $x_{n+1}=n/2$. 

    \begin{align*}
        \delta  & > \|G_5(z^*)-x \|_p \geq |G_5(z^*)_{n+1}-x_{n+1} | \\
        &=|s_{n+1}-\frac{n}{2} |  =|t_{n+1}-\frac{n}{2} |  \\  
        &= |\sum_{i=1}^n \max\{v_i,1/2\} + \sum_{i=1}^n -\min \{ v_i,1/2\} - \frac{n}{2} | \\
        &= \frac{n}{2} - \sum_{i=1}^n (\max\{v_i,1/2\} -\min \{ v_i,1/2\})
    \end{align*}
    
    As a result, we ensure that all $v_i, i\in [n]$  are in $[0,\delta]$ or $[1-\delta,1]$.


    In the $3^{rd}$ layer, the first $n$ entries are $u_{[n]}=\textnormal{ReLU}(W_3 v + b_3)$ where $W_3=-  I_n $, and $b_3=  \frac{1}{2}\cdot\Vec{1} \in \mathbb R^n$. This means that, for all $i \in [n]$, $u_{i} =0$ if $v_i \in [1-\delta,1]$ and $u_{i} \in [1/2-\delta,1/2]$ if $v_i \in [0,\delta]$.

    In the $4^{th}$ layer, $t_{[n]}=\textnormal{ReLU}(W_4 u_{[n]} + b_4)$ where $W_4=- 4\cdot I_n $, and $b_4= \Vec{1} \in \mathbb R^n$. This means that, for all $i \in [n]$, $t_{i} = 1$ if $v_i \in [1-\delta,1]$ and $t_{i} =0$ if $v_i \in [0,\delta]$

    In the $5^{th}$ layer, the input is a vector with ${0,1}$ entries in the first $n$ coordinates. The construction of $5^{th}$ layer is the same as the proof of Theorem \ref{thm_approx_binary}, so the reduction to the binary case is completed.

\end{proof}

\subsection{Reduction from Half-Clique}\label{sec:halfclique}

In the above results, we demonstrated SETH-hardness for the generative inverse problem for $\ell_p$ norm for positive odd integer $p$ through Theorem \ref{thm_approx_binary} and \ref{thm_approx_real}. 

To handle the case of even $p$, we cannot reduce from CVP, following arguments in~\cite{aggarwal2021fine}. Instead, we achieve this through two reductions: the first reduction is from the Half-Clique problem which is introduced in Section \ref{sec:halfclique}, and the second reduction is from the Vertex Cover problem, which is explained in Section \ref{sec:vertexcover}. Both of these approaches enable us to show hardness via a reduction from ETH. 


First, we demonstrate a $2^{\Omega(n)}$ lower bound for binary inputs via a reduction from Half-Clique.

\begin{theorem}
    Assume ETH. There is a 1-layer, $O(n^2)$-width ReLU network $G_1$ and an observation $x$ such that computational complexity to determine whether there exists a $z \in \{ 0,1 \}^n$ with $\|G_1(z)-x\|_p \leq \delta $ is $2^{\Omega(n)}$ for a given $\delta>0$ and any positive even number $p$. 
    \label{thm_approx_l_even_binary}
\end{theorem}

\begin{proof}
    We will make a reduction to a 1-layer ReLU network $G_1$ from the Half-Clique problem on a positive edge-weighted graph $G(V,E)$ where $|V|=n$. The Half-Clique problem is to determine if a half-clique with a total weight less than $M$ exists. We will use the same trick as in the proof of Theorem \ref{thm_approx_binary} that $G_1(z)=\textnormal{ReLU}(W_1z+b_1)$ where ${W_1}=\begin{bmatrix} W \\ -W \end{bmatrix} $ and ${b_1}=\begin{bmatrix} b \\ -b \end{bmatrix} $. Thus, we can consider the problem definition to be $\|Wz+b\|_p \leq \delta$ when the target is the all zeros vector, i.e., $x=\Vec{0}$.
    
    Firstly, label the vertices by $1,\dots,n $ arbitrarily. Construct the matrix $C \in \mathbb 
    R^{ \binom n 2 \times n}$ where each column represents a vertex and each row represents a pair of vertices, and construct the vector $c \in \mathbb R^{\binom n 2}$ as follows. For an edge $e(i,j) \in E$ with edge weight $w_e$, the corresponding values in matrix $C$ and vector $c$ are given by
    $C_{ek}=2\sqrt[\leftroot{-2}\uproot{2}p]{w_{e}}$ if $k\in \{i,j\}$, $C_{ek}=0$ otherwise, and $c_e=-\sqrt[\leftroot{-2}\uproot{2}p]{w_e}$. For a non-edge $e(i,j) \notin E$, the corresponding values in matrix $C$ and vector $c$ are $C_{ek}=2\alpha$ if $k\in \{i,j\}$, $C_{ek}=0$ otherwise, and $c_e=-\alpha$, here $\alpha$ is a large constant which will be determined in the analysis.

    The matrix $W \in \mathbb R^{(\binom n 2 +1) \times n}$ is constructed as the concatenation of $C$ and $\beta\cdot\vec{1}=[\beta, \beta, \dots, \beta] \in \mathbb R^{1 \times n}$, and vector $b \in \mathbb R^{n+1}$ is defined by $b_{[n]}=c$ and $b_{n+1}=-(n/2)\beta$, here $\beta$ is a large constant which will be defined in the analysis. Let $Z$ be the number of non-edges, and the given value $\delta$ be 
    $$
    \delta = \sqrt[\leftroot{-2}\uproot{2}p]{ \sum_{e \in E} w_e + \alpha^p Z  + (3^p-1) M }.$$

    We show the equivalence of the problems as before. 
    
    $\implies$ If there is a half-clique $C$ with weight less than $M$, then consider vector $z \in \{0,1\}^n$ as $i^{th}$ entry is $1$ if and only if $i^{th}$ vertex is in clique $C$. Since it is a half-clique, exactly $n/2$ entries of $z$ is $1$, so the last entry of $Wz+b$ is $0$. All the pairs that appear in the clique have an edge, so that the contribution to the $\|Wz+b\|_p^p$ is $3^pw_e$. For a non-edge, both values in $z$ cannot be $1$, because only the clique vertices get value $1$. In any case, a non-edge contributes $\alpha^p$. Similarly, for an edge that does not appear in the clique, both values in $z$ cannot be $1$. Thus, the contribution will be $w_e$ for each of those edges. As a sum;
    \begin{align*}
      \|Wz+b\|_p^p &= 3^p \sum_{e \in C} w_e+ \alpha^p |\{e \notin E\}| + \sum_{e \in E \setminus C} w_e \\
      &= (3^p-1) \sum_{e \in C} w_e+ \alpha^p Z + \sum_{e \in E } w_e \\
      &< (3^p-1) M + \alpha^p Z + \sum_{e \in E } w_e = \delta^p
    \end{align*}
    which completes one direction.

    $\impliedby$ If there is a $z \in \{0,1\}^n$ such that $\|Wz+b \|_p<\delta$, then we need to show that the corresponding vertices that take $1$ in vector $z$ (say $V_z$) gives a half-clique with weight less than $M$. If there is a non-edge in a pair from $V_z$, the contribution of that pair to $\|Wz+b\|_p^p$ is $3^p\alpha^p$. For any other non-edge, the contribution is at least $\alpha^p$. This means that we have $\|Wz+b\|_p^p\geq 3^p\alpha^p + (Z-1) \alpha^p$. Because of the choice of $\delta$, $$\|Wz+b\|_p^p\geq 3^p\alpha^p + (Z-1) \alpha^p > \delta^p$$ for sufficiently large $\alpha$. So, every pair of nodes in the $V_z$ is connected by an edge, forming a clique. Also, if $|V_z|$ is different than $n/2$, then the contribution of the last entry to $\|Wz+b \|_p^p$ is at least $\beta^p$ which can be selected larger than $\delta^p$. Then, the clique is a half-clique; let's denote it by $C$. Summing, we get:
    \begin{align*}
       \sum_{e \in E} w_e + \alpha^p Z  + (3^p-1) M  &= \delta^p \\ &> \|Wz+b\|_p^p \\ 
      &= 3^p \sum_{e \in C} w_e+ \alpha^p |\{e \notin E\}| + \sum_{e \in E \setminus C} w_e \\
      &= (3^p-1) \sum_{e \in C} w_e+ \alpha^p Z + \sum_{e \in E } w_e
    \end{align*}
    Thus, the half-clique C has weight $ \sum_{e \in C} w_e <M$. 

    This completes the reduction.

\end{proof}

Since Half-Clique is a special case of $k$-Clique problem and under the assumption of ETH, the computational complexity of Clique is $2^{\Omega(n)}$ \cite{lokshtanov2013lower}. Our reduction above implies ETH hardness for approximately inverting models. 

We also present hardness results for dealing with real valued latent inputs in Theorem \ref{thm_approx_l_even_real}. To transition from binary to real numbers, the same approach as used in the proof of Theorem \ref{thm_approx_real} can be followed. Specifically, by constructing a ReLU network with $5$ layers, the problem can be reduced to the binary case whose ETH hardness is proved in Theorem \ref{thm_approx_l_even_binary}. We omit its proof for brevity.




\begin{theorem}
    Assume ETH. There is a 5-layer, $O(n^2)$-width ReLU network $G_5$ and an observation $x$ such that computational complexity to determine whether there exists a $z \in \{ 0,1 \}^n$ with $\|G_5(z)-x\|_p \leq \delta $ is $2^{\Omega(n)}$ for a given $\delta>0$ and any positive even number $p$. 
    \label{thm_approx_l_even_real}
\end{theorem}

\subsection{Reduction from Vertex Cover}\label{sec:vertexcover}

In Section \ref{sec:halfclique}, we showed $2^{\Omega(n)}$ hardness of approximate generative inversion by reduction from Half-Clique. Here, we give the following proof of Theorem \ref{thm_approx_l_even_binary} by reduction from Vertex Cover.

\begin{proof}
    Consider the graph $G=(V,E)$. The Vertex Cover problem asks if $G(V,E)$ has a vertex cover of size $q$. Counting the number of edges takes $O(n^2)$ time. Let $Z$ denote the number of edges. Construct a matrix $C \in \mathbb R^{\binom{n}{2} \times n}$ such that each row represents a pair of vertices and each column represents a vertex, also construct a vector $c\in \mathbb R^{\binom{n}{2}}$ as follows: For an edge $e(i,j) \in E$, put $2\alpha$ to the corresponding edge-vertex intersection of $C$, i.e. $C_{ei}=C_{ej}=2\alpha$, and put $-\alpha$ to the $c_e$. All the other entries in the matrix $C$ and the vector $c$ are set to be $0$.

    Matrix $W \in \mathbb R^{(\binom n 2 +1) \times n}$ is constructed as the concatenation of $C$ and $\beta\cdot\vec{1}=[\beta, \beta, \dots, \beta] \in \mathbb R^{1 \times n}$, and vector $b \in \mathbb R^{n+1}$ is defined by $b_{[n]}=c \in \mathbb R^n$ and $b_{n+1}=-(n-q)\beta$, here $\beta$ is a large constant which will be defined in the analysis.    Let the given value $\delta$ be $\sqrt[\leftroot{-2}\uproot{2}p]{Z \alpha^p}$.

    We show the equivalence of the problems as before. 
    
    $\implies$ Suppose that there is a vertex cover $D$ with size $q$. Then, consider vector $z^* \in \mathbb \{0,1\}^n$ such that $i^{th}$ entry is $0$ if and only if $D$ includes $i^{th}$ vertex.
    
    For each edge $e(i,j) \in E$, at least one of the vertices belongs to $D$. Then, for this edge $e(i,j)$, either $z_i^*$ or $z_j^*$ is $0$. In any case the contribution to the $\|Wz^*+b\|_p^p $ is $\alpha^p$. For a non-edge, all the values on the row are $0$. Thus,

    $$\|Wz^*+b\|_p^p = \sum_{e \in E} \alpha^p = Z\alpha^p =\delta^p$$
    which completes one direction.

    $\impliedby$  Suppose that there is a $z \in \{0,1\}^n$ such that $\|Wz+b \|_p \leq \delta$, then we need to show that the corresponding vertices that take $0$ in vector $z$ (say $D$) gives a vertex cover with size $q$.

    Assume that there is an edge connecting a pair from $V \setminus D$, the contribution of that pair to $\|Wz+b\|_p^p$ is $3^p\alpha^p$. For any edge, the contribution is at least $\alpha^p$. It means that we have $\|Wz+b\|_p^p\geq 3^p\alpha^p + (Z-1) \alpha^p > Z \alpha^p = \delta^p$, which gives a contradiction. So, every edge has a node in $D$, forming a vertex cover. Also, if $|D|$ is different than $q$, then the contribution of the last entry to $\|Wz+b \|_p^p$ is at least $\beta^p$ which can be selected larger than $\delta^p$. Then, it signifies that there exists a vertex cover of size $q$.
    This completes the reduction. 
\end{proof}





\newpage
\bibliographystyle{IEEEtran}
\bibliography{refs.bib}

\appendix

\section*{Appendix}
\section{Proof of Theorem \ref{thm_approx_real} for General Case} \label{proof_thm4_any_delta} 
\begin{proof}
    In the following proof, $c$ is a large constant.
    
    The first 2 layers of $G(\cdot)$ are designed to force the entries of $z$ to lie in $[0,c\delta]$. For this purpose $z_i$ is mapped to $v_i=\min\{ \max \{ z_i,0\},c\delta\}$ for all $i \in [n]$, as in proof of Theorem \ref{thm_exact_real}.

    In the $3^{rd}$ layer, the first $n$ nodes are defined by the following formula:
    $$u_{[n]}=\textnormal{ReLU}(W_3 v + b_3)$$ where $W_3= -  I_n $ and $b_3= (c-1)\delta\cdot\Vec{1} \in \mathbb R^n$. Also, introduce 2 more nodes: $$u_{n+1}=\sum_{i=1}^n \max\{v_i,c\delta/2\},  u_{n+2}=\sum_{i=1}^n -\min \{ v_i,c\delta/2\}.$$

    In the $4^{th}$ layer, the first $n$ nodes are defined by $t_{[n]}=\textnormal{ReLU}(W_4 u_{[n]} + b_4)$ where $W_4=-\cdot I_n $, and $b_4= \Vec{1} \in \mathbb R^n$. Also, we add another node 
    $$t_{n+1}=u_{n+1}+ u_{n+2}= \sum_{i=1}^n \max\{v_i,c\delta/2\} + \sum_{i=1}^n -\min \{ v_i,c\delta/2\}.$$

    In the last layer, the construction is similar to the proof of Theorem \ref{thm_approx_binary} with an addition of one more node given by $s_{m+1}= t_{n+1}$ so that the output is $s \in \mathbb R^{m+1}$. The given output $x$ is $x=[\Vec{0},n/2]\in \mathbb R^{m+1}$. 

    We need to show two facts: 1) A YES instance for CVP implies a YES for generative model inversion, and 2) YES for generative inversion implies YES instance for CVP.

    $\impliedby$ The first part is straightforward as in Theorem \ref{thm_approx_binary}. 
    
    $\implies$ For the second part, suppose there is a $z^* \in R^n$ such that $\|G_5(z^*)-x \|_p < \delta$. The first 2 layers serve to bound the input $z$, so we take the corresponding bounded output values $v^* \in [0,c\delta]^n$. 
    
    We know that, if a real number $a \in [0,1]$, then $\max\{a,c\delta/2\}-\min\{a,c\delta/2\} \leq c\delta/2$, and the equality holds only when $a=0$ or $a=c\delta$. We are given $x_{n+1}=nc\delta/2$. Then,

    \begin{align*}
        \delta  & > \|G_5(z^*)-x \|_p \geq |G_5(z^*)_{n+1}-x_{n+1} | \\
        &=|s_{n+1}-\frac{n c\delta}{2} |  =|t_{n+1}-\frac{n c\delta}{2} |  \\  
        &= \Big|\sum_{i=1}^n \max\{v_i,c\delta/2\} + \sum_{i=1}^n -\min \{ v_i,c\delta/2\} - \frac{nc\delta}{2} \Big| \\
        &= \frac{nc\delta}{2} - \sum_{i=1}^n (\max\{v_i,c\delta/2\} -\min \{ v_i,c\delta/2\}) \,.
    \end{align*}
    
    As a result, we make sure that all $v_i, i\in [n]$  are in $[0,\delta]$ or $[c\delta-\delta,c\delta]$.

    In the $3^{rd}$ layer, the first $n$ entries are $u_{[n]}=\textnormal{ReLU}(W_3 v + b_3)$ where $W_3=-  I_n $, and $b_3=  (c-1)\delta \cdot\Vec{1} \in \mathbb R^n$. This means that, for all $i \in [n]$, $u_{i} =0$ if $v_i \in [c\delta-\delta,c\delta]$ and $u_{i} \in [(c-2)\delta,(c-1)\delta]$ if $v_i \in [0,\delta]$.

    In the $4^{th}$ layer, $t_{[n]}=\textnormal{ReLU}(W_4 u_{[n]} + b_4)$ where $W_4=-  I_n $, and $b_4= \Vec{1} \in \mathbb R^n$. This means that for all $i \in [n]$, $t_{i} = 1$, if $v_i \in [c\delta-\delta,c\delta]$ and $t_{i} =0$ if $v_i \in [0,\delta]$.

    In the $5^{th}$ layer, the input is a vector with ${0,1}$ entries in the first $n$ coordinates. The construction of $5^{th}$ layer is the same as the proof of Theorem \ref{thm_approx_binary}. Therefore, the reduction to the binary case is complete.

\end{proof}

\end{document}